\documentclass{article}

\usepackage[nonatbib,preprint]{neurips_2021}
\usepackage[square,numbers,compress]{natbib}
\usepackage[utf8]{inputenc} 
\usepackage[T1]{fontenc}    
\usepackage{url}            
\usepackage{booktabs}       
\usepackage{amsfonts}       
\usepackage{nicefrac}       
\usepackage{microtype}      
\usepackage{xcolor}         
\usepackage{amsthm}
\usepackage{amsmath}
\usepackage{graphicx}

\definecolor{linkcolor}{RGB}{74, 102, 146}
\usepackage[colorlinks=true,allcolors=linkcolor,pageanchor=true,plainpages=false,pdfpagelabels,bookmarks,bookmarksnumbered]{hyperref}

\usepackage[nameinlink]{cleveref}
\Crefname{section}{Sect.}{Sects.}
\Crefname{appendix}{App.}{Apps.}
\Crefname{proposition}{Prop.}{Props.}

\definecolor{jr_comment}{RGB}{63, 161, 43}

\definecolor{bdacolor}{RGB}{168, 141, 201}

\title{Input Convex Gradient Networks}

\author{%
  Jack Richter-Powell  \\
  McGill University\\
  \texttt{jack.richter-powell@mail.mcgill.ca} \\
  \And
  Jonathan Lorraine\\
  University of Toronto\\
  \texttt{lorraine@cs.toronto.edu}
  \And
  Brandon Amos\\
  Facebook AI Research\\
  \texttt{bda@fb.com}
}

\begin{document}

\newtheorem{theorem}{Theorem}
\newtheorem{proposition}{Proposition}
\newtheorem{definition}{Definition}

\newcommand{\ep}{\varepsilon}
\newcommand{\N}{\mathbb{N}}
\newcommand{\R}{\mathbb{R}}
\newcommand{\Z}{\mathbb{Z}}
\newcommand{\Q}{\mathbb{Q}}
\newcommand{\cF}{\mathcal{F}}
\newcommand{\Ex}[1]{\underset{#1}{\mathbb{E}}}
\newcommand{\fa}[2]{\forall #1 \in #2}
\renewcommand{\div}{\text{div}}
\newcommand{\ang}[1]{\left\langle #1 \right\rangle}
\newcommand{\norm}[1]{\left\lVert#1\right\rVert}
\newcommand{\argmin}[1]{\underset{#1}{\text{argmin}} \hspace{6 pt}}
\newcommand{\argmax}[1]{\underset{#1}{\text{argmax}} \hspace{6 pt}}
\newcommand{\transpose}{\top}
\newcommand{\identity}{\boldsymbol{I}}


\maketitle

\begin{abstract}
  The gradients of convex functions are expressive models
  of non-trivial vector fields.
  For example, Brenier's theorem yields that the optimal transport map between any two measures on Euclidean space
  under the squared distance is realized as a convex gradient,
  which is a key insight used in recent generative flow models.
  In this paper, we study how to model convex gradients by integrating
  a Jacobian-vector product parameterized by a neural network,
  which we call the \emph{Input Convex Gradient Network (ICGN)}.
  We theoretically study ICGNs and compare them to
  taking the gradient of an
  Input-Convex Neural Network (ICNN), empirically demonstrating that a single layer ICGN can fit a toy example better than a single layer ICNN. Lastly, we explore extensions to deeper networks and connections to constructions from Riemannian geometry.
\end{abstract}

\section{Introduction}
Optimal Transport has seen an explosion
of computational interest within the machine learning community
over the last decade.
For example, the Wasserstein metric enables loss functions to
leverage the geometry of the underlying space by allowing one to
lift any ground cost on a Polish space to the space of measures in a
way that metrizes weak convergence.
Several works have explored using this distance computationally, such as in
Entropic Optimal Transport
\cite{genevay2017learning,salimans2018improving},
or in Kantorovich-Rubinstein duality for the $W_1$ cost
\cite{arjovsky2017wasserstein,gulrajani2017improved,gemici2018primaldual}.
The machine learning community has also recently been
interested in applications of
Brenier's theorem~\citep{brenier1991polar,villani2003topics},
which states that the optimal map for the $W_2$ problem is realized as
the gradient of a convex function which maximizes the Kantorovich dual problem.
This motivates the use of convex gradients for problems such as density
estimation and generative modeling, since any distribution with finite second
moment can be realized as the pushforward of a source density by a convex gradient.

In practice, however, it is difficult to expressively model the gradients of convex functions.
The leading approach -- the Input Convex Neural Network (ICNN)
\cite{amos2017input}  -- models a convex potential which
can be differentiated with respect to the inputs to produce a gradient map.
\citet{huang2021flows} combine Brenier's theorem with the ICNN gradients
to design flow based density estimators, and
\citet{makkuva2020optimal,korotin2019wasserstein} use a similar combination to
solve high-dimensional barycenter and transport problems.
While \citet{huang2021flows} prove a universal approximation theorem
for the ICNN, the result relies on stacking a large number of layers.
This detail is not just theoretical; in Section~\ref{exp:poly_example}, we give
a polynomial where a 1-layer ICNN struggles to fit its gradient.

The core difficulty with differentiating a neural network to model
a gradient is that a product structure emerges \citep{saremi2019approximating}. This happens because the chain rule turns the composition of layers into a product of their corresponding Jacobians.
This does not cause issues for training the network on objectives
involving the scalar output, like regression, but can become problematic for objectives involving the gradient of the network's output
\citep{huang2021flows,makkuva2020optimal,korotin2019wasserstein}.
Intuitively, the product of layers of a neural network has similarities to a polynomial, and can suffer from oscillations related to the Runge phenomena -- see ~\citep{metz2021gradients}.
Instead, it would be desirable to directly model the gradient in a way closer to a feedforward network, while guaranteeing that the network parameterizes a convex gradient.

\subsection{Our contributions}

  

For these reasons, we introduce a new class of models which we refer to as \textit{Input Convex Gradient Networks} (ICGN), which are implicitly parameterized by operations on the output of a "hidden" network, similar to the Neural ODE, \cite{chen2018neural}, or Deep Equilibrium Model, \cite{bai2019deq}.
Specifically, we perform a numerical line integral on a symmeterization of the Jacobian, derived from the Gram decomposition of a symmetric Positive Semi-Definite (PSD) matrix; see \eqref{eq:int_model}).
We take this indirect -- and potentially more complex -- approach for 2 reasons:
\begin{enumerate}
  \item Modeling the Jacobian directly allows us to enforce constraints on its structure.
  In Theorem \ref{thm:cvx_pot}, we show our constraints guarantee the model parameterizes a convex gradient. Due to our interpretation via the chain rule, we also believe this avoids having the ill-behaved product structure that the differentiated ICNN suffers.
  \item Implicitly modeling then integrating allows us to succinctly build complex models by leveraging the inherent complexity of integration. In this sense, we view our work as a compromise between full ODE models like \cite{chen2018neural} and a regular feedforward architecture.
\end{enumerate}

The last point relates to what we conjecture is a key benefit of our work.
We create a model by integrating a "hidden" network, with which we hope to be able to compactly parameterize complex models.
By contrast, existing work models gradients by differentiating ICNNs, where the gradient has a similar order of expressiveness to the ICNN itself and is potentially ill-behaved.


As is often said in introductory calculus:
"Differentiation is mechanics, but integration is art."
Although the gradient of an ICNN (or more generally any network) is related to the network itself by a relatively simple procedure symbolically~\citep{griewank2008evaluating} -- the same cannot be said for the integral.
It is exactly this lack of a simple relationship we hope will allow us to compactly parameterize complex models with the ICGN.



\subsection{Outline}
In \S~\ref{sec:convex-bg} we motivate the structure of our model, and give an interpretation in terms of the chain rule. In \S~\ref{sec:convex-int} we describe an operation that computes the line integral of a symmeterized Jacobian, and give conditions for when it produces a convex gradient.
In \S~\ref{sec:ICGN} we introduce our model, which applies this operation to a suitable neural network.
In \S~\ref{sec:expr} we provide an implementation and experiments on a toy problem.
We round off \S~\ref{sec:open} by discussing some open problems and directions for future work.

\subsection{Related work}

\textbf{Structured higher-order info:} 
Various applications require various guarantees about learned functions.
    We are primarily interested in convexity, which ICNNs~\citep{amos2017input} guarantee.
    \cite{anil2019sorting} looks at guaranteeing Lipschitzness, while \cite{pitis2020inductive} looks at guaranteeing the function is a valid distance.
    Sobolev networks \cite{czarnecki2017sobolev} take the alternative approach of fitting higher-order info with a soft penalty.
    In future work, we are interested in experimenting with enforcing properties other than convexity with our method.

\textbf{Implicit Models via integration:}
    Implicit models with no explicit architecture are a powerful tool for generating complex models.
    One way to generate implicit models is when the weights are used in some iterative procedure to generate an output -- ex., deep equilibrium networks~\citep{bai2019deep}.
    We are particularly interested in implicit models generated via integration -- i.e Neural ODE's~\cite{chen2018neural}, which have seen recent popularity in applications like spatio-temporal point processes~\citep{chen2020neural}.
    Jacnet~\cite{lorraine2019jacnet} proposed a related method to structure higher-order information by parameterizing models implicitly with integration.
    


\textbf{Flow-based models:}
    We are motivated by downstream applications of our methods to flow-based models, which are powerful tools for designing probabilistic models with tractable density \cite{huang2018naf}.
    Recently, \cite{huang2020convex} introduced Convex Potential Flows using the gradient map of an ICNN.
    Our method could be applied to generate similar convex flows -- we are excited for this future work.
    
\textbf{Optimal Transport:}
    \citet{makkuva2020optimal} explored using Input Convex Neural Networks for learning transportation maps, while \citet{alvarez2021optimizing} and \citet{mokrov2021largescale} used ICNN's for the Kantorovich dual specifically in the setting of Wasserstein gradient flows \cite{jordan1998variational}. \citet{fan2021scalable} also attempted solving the Wassertein Barycenter \cite{agueh2011barycenters} problem using ICNNs.


\section{Building Convex Gradients}
Below, we give proofs inline when possible, but we defer them to the appendix when they are lengthy.

\subsection{A basic structure theorem for convex gradients}
\label{sec:convex-bg}
We highlight some properties of vector mappings where the Jacobian matrix has a particular structure. Let $G :\R^n \to \R^n$ be a smooth vector field where there exists $V : \R^n \to \R^{n \times n}$ such that 
\begin{gather} DG_x = V^T(x) V(x) \label{eq:gram}\end{gather}
where $DG = \frac{\partial G}{\partial x} : \R^n \to \R^{n \times n}$ is the Jacobian matrix of $G$.
Then the following theorem holds:

\begin{theorem}
\label{thm:cvx_pot}
  For any $G$ that satisfies \eqref{eq:gram} there exists a convex function $g :\R^n \to \R$ such that $G = \nabla g$. i.e $G$ is the gradient of convex function.
\end{theorem}
The proof follows by combining the following 2 propositions:

\begin{proposition}
  There exists $\varphi : \R^n \to \R$ such that 
  $\nabla \varphi = G$.
  i.e, it is a potential field.
\end{proposition}

\begin{proof}
  Treating $G$ as a 1-form by identifying
  $G = \sum_i G_i dx^i$ (which is possible since $T_p\R^n \simeq \R^n$), $G$ is closed if (see equation 11.21 in  \cite{lee2013introduction})
  \[ \frac{\partial G_i}{\partial x^j} = \frac{\partial G_j}{\partial x^i}  \qquad  \forall i,j \]
  This is exactly equivalent to the Jacobian of $G$ being symmetric, since the entries are $[DF]_{ij} = \frac{\partial G_i}{\partial x^j}$. But that follows because the product $V^TV$ is symmetric by construction. Now since the domain is the Euclidean space, by the Poincaré lemma \cite{lee2013introduction}, all closed 1-forms are exact. Thus there exists a 0-form $g$ such that 
  $ dg = G $.
  However, for 0-forms, the exterior derivative just reduces to the gradient (modulo lowering an index) thus $\nabla g = G$ and so $G$ is conservative. 
\end{proof}

\begin{proposition}
  The $g$ that satisfies 
  $\nabla g = G$
  is convex.
\end{proposition}
\begin{proof}
  Since $g$ is smooth, it is sufficient to check that the Hessian of $g$, $\nabla^2 g$ is Positive Semi-Definite (PSD). But since $\nabla g = G$ it follows that
  \[ \nabla^2 g = D[\nabla g] = DG = V^TV \]
  which is PSD because it is a Gram decomposition. Thus $g$ is convex.
\end{proof}

\subsection{Building convex gradients from Gram products}

Given the Gram factorization we explored above, it is tempting to ask the following question: given $G : \R^n \to \R^m$ smooth, does there exist $H: \R^m \to \R^n$ such that 
$DH_{G(x)} = [DG_x]^T $? This question is led by the observation that, by the chain rule, the Jacobian of $H\circ G$ becomes
\[ D(H \circ G)_x = DH_{G(x)}DG_x = [DG_x]^T DG_x\]
so by Proposition 1 \& 2, $H \circ G$ is the gradient of a convex function. In such a case, we say \textit{H convexifies G}. Here are a few examples of these $G,H$ pairs:
\newpage 
\begin{enumerate}
  \item If $G(x) = Ax$ for some matrix $A \in \R^{m \times n}$, then $H$ is given by $H(x) = A^Tx$. 
  \item If $\sigma : \R^n \to \R^n$ is any smooth, invertible and elementwise function, then it's Jacobian is diagonal of the form (recall $x^i$ is the i'th component of $x$)
  \[ D\sigma_x = \begin{pmatrix}\sigma'(x^1) & 0 &\cdots & 0 \\
    0 & \sigma'(x^2) & \cdots & 0\\ \vdots& \vdots&\ddots & 0 \\ 0&0&\cdots &\sigma'(x^n) \end{pmatrix}\]
    so if $G=\sigma$, the $H$ is given by elementwise $\gamma$ such that $\gamma' = \sigma' \circ \sigma^{-1}$
  \item If $A$ is as in 1. and $\sigma$ is as in 2., then $G = \sigma \circ A$ is convexified by $H = A^T \circ \gamma$, where $\gamma$ is such that $\gamma' = \sigma' \circ \sigma^{-1}$. In that case,
  \[ DH_{G(x)}= A^T D\gamma_{G(x)} = A^T D\sigma_{\sigma^{-1} \circ \sigma Ax} = A^T D\sigma_{Ax} =  [DG(x)]^T \]
  this shows that there are non-trivial non-linear solutions.
\end{enumerate}

Unfortunately, composing these solutions further does not yield $G$'s with closed form solutions. The next proposition, however, provides a characterization in the case where $G$ is assumed to be invertible:
\begin{proposition}
  Let $G: \R^n \to \R^n$ be an smooth invertible vector field. Then there exists $H : \R^n \to \R^n$ smooth such that
  \begin{gather*} DH_{G(x)} = [DG_x]^T \qquad \text{  or equivalently  } \qquad DH_{x} = [DG_{G^{-1}(x)}]^T \label{eq:trans} \end{gather*}
  
  if and only if
  \[ \frac{\partial}{\partial x^j} \left( \frac{\partial G_i}{\partial x^k} (G^{-1}(x)) \right) = \frac{\partial}{\partial x^i} \left( \frac{\partial G_j}{\partial x^k} (G^{-1}(x)) \right) \]
  or equivalently, for each 1-form defined by 
  \[ \gamma_k := \sum_{i=1}^n \frac{\partial G_i}{\partial x^k}(G^{-1}) dx^i\] 
  $\gamma_k$ must be closed, i.e $d\gamma_k = 0$. 
  \label{prop3}
\end{proposition}
The proof is given in \ref{apx:proof3}, along with a discussion of connections to Brenier's Polar Factorization theorem.

\subsection{Gram products and integration}
\label{sec:convex-int}
In general, even if we can verify that an invertible $G$ satisfies Proposition 3, there is little hope of finding a closed form for $H$. But given that we know if such an $H$ exists, it must satisfy 
\begin{gather} D(H \circ G)_{x} = [DG_x]^T DG_x \label{eq:cvx}\end{gather}

For our purposes, we are interested in computing the composition $H \circ G$, so we will proceed by integrating the right hand side. The following definition makes this precise, and generalizes to the case where $G$ is not assumed invertible:

\begin{definition}
\label{def:cvxific}
  Let $G : \R^n \to \R^m$ be a smooth vector field. Assume $G$ satisfies the following partial differential equation:
  \begin{gather} \frac{\partial^2 G}{\partial x^k \partial x^i} \cdot \frac{\partial G}{\partial x^j} = \frac{\partial^2 G}{\partial x^k \partial x^j} \cdot \frac{\partial G}{\partial x^i} \qquad \forall 1 \leq i,j,k \leq n
  \label{eq:pde} \end{gather}
  where the $\cdot$ is a Euclidean dot-product, $\frac{\partial G}{\partial x^i}$ is a vector derivative (i.e a column of the Jacobian).
  We define the convexification of $G$ to be
  \begin{gather} F(x) = \int_0^1 [DG_{sx}]^TDG_{sx} x ds \qquad F:\R^n \to \R^n \label{eq:int_model}\end{gather}
  \label{def1}
\end{definition}
The utility of this definition is made clear in the following proposition:
\begin{theorem}
\label{thm:jacfac}
  The Jacobian of $F$, $DF$ takes the form 
  \[ DF = [DG]^T DG \]
  by Theorem-\ref{thm:cvx_pot}, this implies there exists a convex $\varphi : \R^n \to \R$ such that $F = \nabla \varphi$.
\end{theorem}
For the proof see Appendix-\ref{apx:proofthm2}. Lastly we give an important structural result -- a weaker version of the characterization we had in the case where $G$ was invertible, but still sufficient for our needs:
\begin{theorem}
\label{thm:exist}
  Given $G : \R^n \to \R^m$ smooth, if there exists a smooth $H :\R^m \to \R^n$ that convexifies $G$ (in the sense of \ref{def:cvxific}), then $H \circ G = F$ as defined above, i.e
  \[ H \circ G (x) = F(x) = \int_0^1 [DG_{sx}]^TDG_{sx} x ds  \]
\end{theorem}
This result can be interpreted as: \textit{if a solution exists, this integral will find it.} For the proof see Appendix-\ref{apx:proof3}.


\section{Input Convex Gradient Networks}\label{sec:ICGN}
Given a Neural Network $M_\theta:\R^n \to \R^m$ with smooth activations, we can apply the transformation \eqref{eq:int_model} to it. Using automatic differentiation, we can efficiently compute the Jacobian-vector and Jacobian-transpose-vector products required for the integrand. In this case we write 
\begin{gather} N_\theta(x) = \int_0^1 [D(M_\theta)_{sx}]^T D(M_\theta)_{sx} x ds \end{gather}

When the explicit model $M_\theta$ satisfies \eqref{eq:pde}, we refer to the implicit model, $N_\theta$, as an \textit{Input Convex Gradient Network} (ICGN). A few details are in order:
\begin{itemize}
  \item To evaluate $N_\theta(x)$ for a point $x$, we numerically compute the line integral. When considering which quadrature method to use, we have a special constraint --  we need the numerical estimator to be non-deterministic, ruling out methods like Gaussian Quadrature or Simpson's rule. 
  This is required because the optimizer can learn to modify the network in ways that take advantage of the estimator using the same fixed points.
  
  
  \item We can have $m > n$ for $M_\theta :\R^n \to \R^m$ allowing us to augment the output dimension, which we can use to increase expressitivity of the network. Additionally, since we are fixing the line integral's path, we are implicitly enforcing $F(0) = 0$. We could also add an explicit constant for $F(0)$ as a parameter.
  
  \item In general, an arbitrary network $M_\theta$ will not satisfy the PDE \eqref{eq:pde}.
  However, $M_\theta$ satisfies \eqref{eq:pde} in the special case where $M_\theta(x) = \sigma (Ax + b)$ -- i.e., a 0 or 1 layer ICGN depending on interpretation.
  
  However, in this one layer case, because we know $H$ in closed form corresponding to $M_\theta$, the integral is really a proof of concept.
  Still, we believe it is possible to design constraints for deeper networks so that they satisfy \eqref{eq:pde}.
  In this case the integral is necessary, but we defer this discussion to \ref{sec:open}.
  
  \item We can also work with explicit models $M_\theta$ without proving they satisfy \eqref{eq:pde}. Since the path in \ref{eq:int_model} is fixed, the output of the implicit model $N_\theta$ remains well defined. Although we lose the closedness and convexity guarantees, for practical applications they might not be necessary. The result is still a highly compact and expressive model, similar to what was first formulated by \citet{lorraine2019jacnet}.
\end{itemize}

\newpage
\section{Experiments}
\label{sec:expr}
We include a Google Colaboratory notebook for reproduction of the experiments \href{https://colab.research.google.com/drive/1BvFXpouwmZtusB_bJKbBRuCqlcozbWva?usp=sharing}{{\color{blue}here}}.

\subsection{A polynomial example}\label{exp:poly_example}
This experiment compares the sizes of an ICNN and ICGN needed to approximate a given convex gradient.
We approximate the following map as our toy example:
\begin{equation}\label{eq:toy}
    T(x,y) = \begin{pmatrix}4x^3 + \frac{1}{2} y + x \\ 3y - y^2 + \frac{1}{2} x \end{pmatrix}
\end{equation}
This map is the gradient of the function $\varphi(x,y) = x^4 + \frac{x^2}{2} + \frac{xy}{2} + \frac{3y^2}{2} - \frac{y^3}{3}$, which is convex on $[0,1]^2$.
In Figure~\ref{fig1} we display the error between an ICGN with few parameters and a ICNN with many more parameters.
The ICGN has 15 parameters -- it is extremely small, with 0 hidden layers i.e $M_\theta = \sigma(Ax + b)$ and an output dimension of 5.
The ICNN does not learn reasonable functions with 0 hidden layers, so we display a 1-layer ICNN.
We did not find the number of hidden units strongly impacted the 1-layer ICNN performance -- we use 25 hidden units in Figure~\ref{fig1} for 78 total parameters.
By adding a second layer, the ICNN was able to fit the function correctly.

\textbf{Takeaway:} Our model -- the ICGN -- effectively approximates the target map $T$ with far fewer parameters than the ICNN.

\begin{figure}
  \center
  \includegraphics[width=1.0\textwidth]{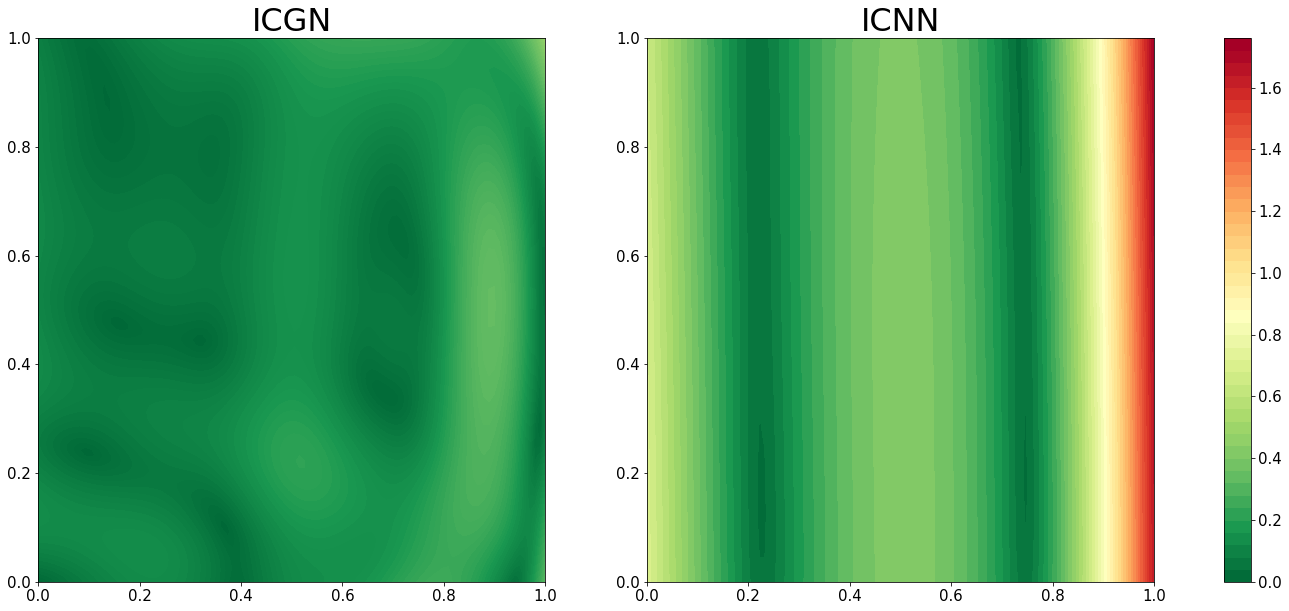}
  \caption{
    We compare methods for learning convex gradients on the map in Equation~\ref{eq:toy}.
    Our method -- the ICGN -- is able to effectively approximate the map with very few parameters, while the moderately sized ICNN struggles. Here the colorscale represents the $L2$ error between the target and the output of each model.
  }
  \label{fig1}
\end{figure}

\section{Future Directions}
\label{sec:open}

Here we list open questions as as interesting directions for future work:
\begin{itemize}
    \item How can we parameterize deeper networks for $M_\theta$ while still ensuring they solve \eqref{eq:pde}? In \S-\ref{sec:open_beltrami} we discuss a connection between our model's structure and Riemannian metric tensors, which we hope might guide our search for a solution.
    \item If using more layers isn't practical (i.e., composition), is there another method of generating more expressive models that still solve \eqref{eq:pde}? We explore the this idea in Appendix~\ref{apx:additive}.
    \item Can we use a different product structure than the Gram product (i.e $V^T V$)? Other options could be to use either a Hadamard product, or a Kronecker product. We explore this in Appendix~\ref{apx:prod}.
    \item Are there applications for this implicit-integration model besides convexity?
    For example, diffusion modelling.
    We believe there is unexplored potential in using integration for modeling.
\end{itemize}

\subsection{Deeper networks and Pullback Metrics}
\label{sec:open_beltrami}
The ICGNs Achilles heel is currently that although a 1-layer network satisfies \eqref{eq:pde}, a network with more layers does not. Still, one layer working at all suggests that it might be possible to extend to deeper networks, given the right constraints on the layers.
Choosing these constraints, however, could be difficult.

Fortunately, the integrand we constructed in the definition of \eqref{eq:int_model}, $[DG]^T DG$, has been analyzed before, and no doubt is well known to geometers -- it is the \textit{pullback} of the flat metric on $\R^n$ via $G$, which would usually be denoted $G^* I$, where $I$ denotes the standard Euclidean metric, i.e $I(v,w) = v^Tw = v \cdot w$, which is represented as a matrix in canonical coordinates is the identity. So asking whether $G$ satisfies the PDE \eqref{eq:pde} is equivalent to asking, \textit{when is the pullback metric $G^*I$ given by a convex potential?} I.e, when does there exists a convex $\varphi$ such that $\nabla^2 \varphi = G^* I$, where $\nabla^2$ is the Riemannian Hessian \cite{lee2019riemann} on a manifold $(M,g)$. 

This was studied before in \citet{shima1976metrichess}, who gave conditions for a Riemannian manifold $(M,g)$ such that the metric tensor is the Hessian $\nabla^2$ of some $\varphi : M \to \R$, where $\nabla$ is a given flat connection. This can be interpreted as the real valued case of a \textit{Kahler metric}. Unfortunately, these conditions do not give a simple answer as to how we should constrain the layers of $M_\theta$ such that the condition \eqref{eq:pde} holds, but nonetheless they provide further evidence that it should be possible. We hope to explore this in future work. 


\section{Conclusion}
    We introduced a method for modelling convex gradients by integrating Jacobian-vector products parameterized by a neural network -- the Input Convex Gradient Network (ICGN).
    We provided theoretical results characterizing what we know about how different flavours of our method will work.
    We demonstrated initial empirical results with our method showing that we can learn non-trivial vector fields with fewer parameters than competing methods.
    Finally, we presented various exciting directions for future work.

    
    
    
    
    


    \newpage
        \bibliography{references}

\begin{thebibliography}{33}
\providecommand{\natexlab}[1]{#1}
\providecommand{\url}[1]{\texttt{#1}}
\expandafter\ifx\csname urlstyle\endcsname\relax
  \providecommand{\doi}[1]{doi: #1}\else
  \providecommand{\doi}{doi: \begingroup \urlstyle{rm}\Url}\fi

\bibitem[Genevay et~al.(2017)Genevay, Peyré, and Cuturi]{genevay2017learning}
Aude Genevay, Gabriel Peyré, and Marco Cuturi.
\newblock Learning generative models with sinkhorn divergences, 2017.

\bibitem[Salimans et~al.(2018)Salimans, Zhang, Radford, and
  Metaxas]{salimans2018improving}
Tim Salimans, Han Zhang, Alec Radford, and Dimitris Metaxas.
\newblock Improving gans using optimal transport, 2018.

\bibitem[Arjovsky et~al.(2017)Arjovsky, Chintala, and
  Bottou]{arjovsky2017wasserstein}
Martin Arjovsky, Soumith Chintala, and Léon Bottou.
\newblock Wasserstein gan, 2017.

\bibitem[Gulrajani et~al.(2017)Gulrajani, Ahmed, Arjovsky, Dumoulin, and
  Courville]{gulrajani2017improved}
Ishaan Gulrajani, Faruk Ahmed, Martin Arjovsky, Vincent Dumoulin, and Aaron
  Courville.
\newblock Improved training of wasserstein gans, 2017.

\bibitem[Gemici et~al.(2018)Gemici, Akata, and Welling]{gemici2018primaldual}
Mevlana Gemici, Zeynep Akata, and Max Welling.
\newblock Primal-dual wasserstein gan, 2018.

\bibitem[Brenier(1991)]{brenier1991polar}
Yann Brenier.
\newblock Polar factorization and monotone rearrangement of vector-valued
  functions.
\newblock \emph{Communications on pure and applied mathematics}, 44\penalty0
  (4):\penalty0 375--417, 1991.

\bibitem[Villani(2003)]{villani2003topics}
C.~Villani.
\newblock \emph{Topics in Optimal Transportation}.
\newblock Graduate studies in mathematics. American Mathematical Society, 2003.
\newblock ISBN 9780821833124.

\bibitem[Amos et~al.(2017)Amos, Xu, and Kolter]{amos2017input}
Brandon Amos, Lei Xu, and J~Zico Kolter.
\newblock Input convex neural networks.
\newblock In \emph{International Conference on Machine Learning}, pages
  146--155. PMLR, 2017.

\bibitem[Huang et~al.(2020{\natexlab{a}})Huang, Chen, Tsirigotis, and
  Courville]{huang2021flows}
Chin{-}Wei Huang, Ricky T.~Q. Chen, Christos Tsirigotis, and Aaron~C.
  Courville.
\newblock Convex potential flows: Universal probability distributions with
  optimal transport and convex optimization.
\newblock \emph{CoRR}, abs/2012.05942, 2020{\natexlab{a}}.
\newblock URL \url{https://arxiv.org/abs/2012.05942}.

\bibitem[Makkuva et~al.(2020)Makkuva, Taghvaei, Oh, and
  Lee]{makkuva2020optimal}
Ashok Makkuva, Amirhossein Taghvaei, Sewoong Oh, and Jason Lee.
\newblock Optimal transport mapping via input convex neural networks.
\newblock In \emph{International Conference on Machine Learning}, pages
  6672--6681. PMLR, 2020.

\bibitem[Korotin et~al.(2019)Korotin, Egiazarian, Asadulaev, Safin, and
  Burnaev]{korotin2019wasserstein}
Alexander Korotin, Vage Egiazarian, Arip Asadulaev, Alexander Safin, and Evgeny
  Burnaev.
\newblock Wasserstein-2 generative networks.
\newblock \emph{arXiv preprint arXiv:1909.13082}, 2019.

\bibitem[Saremi(2019)]{saremi2019approximating}
Saeed Saremi.
\newblock On approximating $\nabla f$ with neural networks.
\newblock \emph{arXiv preprint arXiv:1910.12744}, 2019.

\bibitem[Metz et~al.(2021)Metz, Freeman, Schoenholz, and
  Kachman]{metz2021gradients}
Luke Metz, C~Daniel Freeman, Samuel~S Schoenholz, and Tal Kachman.
\newblock Gradients are not all you need.
\newblock \emph{arXiv preprint arXiv:2111.05803}, 2021.

\bibitem[Chen et~al.(2018)Chen, Rubanova, Bettencourt, and
  Duvenaud]{chen2018neural}
Ricky~TQ Chen, Yulia Rubanova, Jesse Bettencourt, and David Duvenaud.
\newblock Neural ordinary differential equations.
\newblock In \emph{Proceedings of the 32nd International Conference on Neural
  Information Processing Systems}, pages 6572--6583, 2018.

\bibitem[Bai et~al.(2019{\natexlab{a}})Bai, Kolter, and Koltun]{bai2019deq}
Shaojie Bai, J.~Zico Kolter, and Vladlen Koltun.
\newblock Deep equilibrium models.
\newblock \emph{CoRR}, abs/1909.01377, 2019{\natexlab{a}}.
\newblock URL \url{http://arxiv.org/abs/1909.01377}.

\bibitem[Griewank and Walther(2008)]{griewank2008evaluating}
Andreas Griewank and Andrea Walther.
\newblock \emph{Evaluating derivatives: principles and techniques of
  algorithmic differentiation}.
\newblock SIAM, 2008.

\bibitem[Anil et~al.(2019)Anil, Lucas, and Grosse]{anil2019sorting}
Cem Anil, James Lucas, and Roger Grosse.
\newblock Sorting out lipschitz function approximation.
\newblock In \emph{International Conference on Machine Learning}, pages
  291--301. PMLR, 2019.

\bibitem[Pitis et~al.(2020)Pitis, Chan, Jamali, and Ba]{pitis2020inductive}
Silviu Pitis, Harris Chan, Kiarash Jamali, and Jimmy Ba.
\newblock An inductive bias for distances: Neural nets that respect the
  triangle inequality.
\newblock \emph{arXiv preprint arXiv:2002.05825}, 2020.

\bibitem[Czarnecki et~al.(2017)Czarnecki, Osindero, Jaderberg, {\'S}wirszcz,
  and Pascanu]{czarnecki2017sobolev}
Wojciech~Marian Czarnecki, Simon Osindero, Max Jaderberg, Grzegorz
  {\'S}wirszcz, and Razvan Pascanu.
\newblock Sobolev training for neural networks.
\newblock \emph{arXiv preprint arXiv:1706.04859}, 2017.

\bibitem[Bai et~al.(2019{\natexlab{b}})Bai, Kolter, and Koltun]{bai2019deep}
Shaojie Bai, J~Zico Kolter, and Vladlen Koltun.
\newblock Deep equilibrium models.
\newblock \emph{arXiv preprint arXiv:1909.01377}, 2019{\natexlab{b}}.

\bibitem[Chen et~al.(2020)Chen, Amos, and Nickel]{chen2020neural}
Ricky~TQ Chen, Brandon Amos, and Maximilian Nickel.
\newblock Neural spatio-temporal point processes.
\newblock \emph{arXiv preprint arXiv:2011.04583}, 2020.

\bibitem[Lorraine and Hossain(2019)]{lorraine2019jacnet}
Jonathan Lorraine and Safwan Hossain.
\newblock Jacnet: Learning functions with structured jacobians.
\newblock \emph{ICML INNF Workshop}, 2019.

\bibitem[Huang et~al.(2018)Huang, Krueger, Lacoste, and
  Courville]{huang2018naf}
Chin-Wei Huang, David Krueger, Alexandre Lacoste, and Aaron Courville.
\newblock Neural autoregressive flows, 2018.

\bibitem[Huang et~al.(2020{\natexlab{b}})Huang, Chen, Tsirigotis, and
  Courville]{huang2020convex}
Chin-Wei Huang, Ricky~TQ Chen, Christos Tsirigotis, and Aaron Courville.
\newblock Convex potential flows: Universal probability distributions with
  optimal transport and convex optimization.
\newblock \emph{arXiv preprint arXiv:2012.05942}, 2020{\natexlab{b}}.

\bibitem[Alvarez-Melis et~al.(2021)Alvarez-Melis, Schiff, and
  Mroueh]{alvarez2021optimizing}
David Alvarez-Melis, Yair Schiff, and Youssef Mroueh.
\newblock Optimizing functionals on the space of probabilities with input
  convex neural networks.
\newblock \emph{arXiv preprint arXiv:2106.00774}, 2021.

\bibitem[Mokrov et~al.(2021)Mokrov, Korotin, Li, Genevay, Solomon, and
  Burnaev]{mokrov2021largescale}
Petr Mokrov, Alexander Korotin, Lingxiao Li, Aude Genevay, Justin Solomon, and
  Evgeny Burnaev.
\newblock Large-scale wasserstein gradient flows, 2021.

\bibitem[Jordan et~al.(1998)Jordan, Kinderlehrer, and
  Otto]{jordan1998variational}
Richard Jordan, David Kinderlehrer, and Felix Otto.
\newblock The variational formulation of the fokker--planck equation.
\newblock \emph{SIAM journal on mathematical analysis}, 29\penalty0
  (1):\penalty0 1--17, 1998.

\bibitem[Fan et~al.(2021)Fan, Taghvaei, and Chen]{fan2021scalable}
Jiaojiao Fan, Amirhossein Taghvaei, and Yongxin Chen.
\newblock Scalable computations of wasserstein barycenter via input convex
  neural networks, 2021.

\bibitem[Agueh and Carlier(2011)]{agueh2011barycenters}
Martial Agueh and Guillaume Carlier.
\newblock Barycenters in the wasserstein space.
\newblock \emph{SIAM Journal on Mathematical Analysis}, 43\penalty0
  (2):\penalty0 904--924, 2011.

\bibitem[Lee(2013)]{lee2013introduction}
J.M. Lee.
\newblock \emph{Introduction to Smooth Manifolds}.
\newblock Graduate Texts in Mathematics. Springer New York, 2013.
\newblock ISBN 9780387217529.

\bibitem[Lee(2019)]{lee2019riemann}
J.M. Lee.
\newblock \emph{Introduction to Riemannian Manifolds}.
\newblock Graduate Texts in Mathematics. Springer International Publishing,
  2019.
\newblock ISBN 9783319917559.
\newblock URL \url{https://books.google.ca/books?id=VUOCDwAAQBAJ}.

\bibitem[Shima(1976)]{shima1976metrichess}
Hirohiko Shima.
\newblock {On certain locally flat homogeneous manifolds of solvable Lie
  groups}.
\newblock \emph{Osaka Journal of Mathematics}, 13\penalty0 (2):\penalty0 213 --
  229, 1976.
\newblock \doi{ojm/1200769511}.
\newblock URL \url{https://doi.org/}.

\bibitem[Horn and Johnson(1994)]{horn1994matrix}
R.A. Horn and C.R. Johnson.
\newblock \emph{Topics in Matrix Analysis}.
\newblock Cambridge University Press, 1994.
\newblock ISBN 9781107392953.
\newblock URL \url{https://books.google.ca/books?id=ukd0AgAAQBAJ}.

\end{thebibliography}
    \clearpage
    \newpage

    \appendix
    

\section{Proof of Prop 3}
\label{apx:proof3}
Here we give the proof of proposition 3:

\begin{proof}
\noindent
\text{ }\\
\vspace{6 pt}
$\implies$ \\
Assume there exists $H$ that satisfies \eqref{eq:trans}. Then since each row of $DH$ is the gradient of a component of $H$, they are exact as 1-forms. Thus 
\[ \gamma_k = \sum_{i=1}^n \frac{\partial G_i(G^{-1})}{\partial x^k} dx^i = \sum_{i=1}^n \frac{\partial H_k}{\partial x^i} dx^i \implies d\gamma_k = 0\]
and by Prop in \citet{lee2013introduction}, the 1-forms $\gamma_k$ are exact on $\R^d$ if and only if in coordinates,
\[ \frac{\partial \gamma_k^j}{\partial x^i}  - \frac{\partial \gamma_k^i}{\partial x^j} = 0 : \forall i,j \]
where $\gamma_k^i$ denotes the $i$th coefficient (function) of $\gamma_k$. Expanding this expression yields
\[ \frac{\partial }{\partial x^i}\left( \frac{\partial G_i(G^{-1})}{\partial x^k} \right)  - \frac{\partial} {\partial x^j}\left( \frac{\partial G_i(G^{-1})}{\partial x^k} \right) = 0 \]
\vspace{6 pt}
\text{} \\ 

$\impliedby$ 
Conversely, assume the 1-forms are closed for all $k$. By the Poincare lemma \cite{lee2013introduction}, closed 1-forms on $\R^d : d \geq 2$ are exact, so there exist 0-forms, $G_k$ such that $dG_k = \nabla G_k = \gamma_k$. But then setting 
\[ G = \begin{pmatrix} G_1 \\ \vdots \\ G_d \end{pmatrix}\]
we have 
\[ DG = [DG_{G^{-1}(x)}]^T \]
as desired.

\end{proof}

In the case where $G$ is not invertible, however, it is harder to say whether such an $H$ exists. We leave exploring this to future work.

\subsection{Note about polar factorization theorem}
Brenier's polar factorization theorem\cite{brenier1991polar}  tells us that 
\begin{theorem}
If $G \in L^2(\R^d ; \R^d)$ is a vector valued $L_2$ mapping, and
\[ \lambda, \nu \in \mathcal{P}_2(\R^d) \qquad \mu := G_\# \lambda \]
are two probability measures with finite second moments, then there exists a convex function $\varphi$ and a map $S : \R^d \to \R^d$ such that
\[ G \circ S =  \nabla \varphi \qquad (\nabla \varphi)_\#(\nu) = \mu \qquad S_\#(\lambda) = \nu \]
and S is the unique $L_2$ projection of $G$ onto the set of maps that pushforward $\lambda$ onto $\nu$ i.e 
\[ S = \argmin{\sigma \in S(\lambda,\nu)} \int_{\R^d} |\sigma(x) - H(x)|^2 \qquad S(\lambda, \nu) = \{ \sigma \in L^2(\R^d;\R^d) | \sigma_\#(\lambda) = \nu \}\]
\end{theorem}

Given the similarities between the structure of this theorem -- the composition of $G \circ S$ versus our construction of finding $H$ such that $H \circ G$ is a convex gradient  -- we naturally question whether our construction is some sort of special case. The question can be posed as, given $G$ that satisfies \eqref{eq:pde}, is there a measure $\nu$ such that $S$ is smooth and $DS_{G(x)} = DG^T_x$ ?  Then $G$ convexifies $S$ in the sense of \ref{eq:cvx}. We hope to explore this connection in future work.


\subsection{Open questions about convex functions}
\label{apx:expressivity}
An interesting question to consider is the following, given any smooth convex $\varphi :\R^d \to \R$, does there exist a pair of vector fields $G,H$, where $H$ convexifies $G$ and 
\[ H \circ G = \nabla \varphi\]

\section{Proof of Theorem 2}
\label{apx:proofthm2}
Here, we prove Theorem 2:
\begin{theorem}
  The Jacobian of $F$, $DF$ takes the form 
  \[ DF = [DG]^T DG \]
  by Theorem 1, this implies there exists a convex $\varphi : \R^n \to \R$ such that $F = \nabla \varphi$.
\end{theorem}
\begin{proof}
  To start, we derive the PDE \eqref{eq:pde}. This equation comes from interpreting each row of $[DG_x]^T DG_x$ as a 1-form, then checking the closedness condition.\footnote{Or alternatively, interpreting the 1-forms as vector valued 0-forms.} Precisely, for the $k$th row, we have

  \[ \left[[DG_x]^T DG_x\right]_k = \begin{pmatrix} \frac{\partial G}{\partial x^k} \cdot \frac{\partial G}{\partial x^1} & \cdots & \frac{\partial G}{\partial x^k} \cdot \frac{\partial G}{\partial x^i} & \cdots &\frac{\partial G}{\partial x^k} \cdot \frac{\partial G}{\partial x^d}\end{pmatrix}\]
  so interpreting this as a 1-form leads to
  \[ \omega_k = \sum_{i=1}^d  \left( \frac{\partial G}{\partial x^k} \cdot \frac{\partial G}{\partial x^i} \right) dx^i \] 
  so then taking the exterior derivative yields 
  \begin{align*} d\omega_k &= \sum_{i=1}^n d\left( \frac{\partial G}{\partial x^k} \cdot\frac{\partial G}{\partial x^i} \right) dx^i \\
  &= \sum_{i=1}^n \sum_{j=1}^n \frac{\partial}{\partial x^j} \left( \frac{\partial G}{\partial x^k} \cdot\frac{\partial G}{\partial x^i} \right) dx^j \wedge dx^i\\
  &= \sum_{i=1}^n \sum_{j=1}^n \left[\left( \frac{\partial^2 G}{\partial x^j \partial x^k} \right) \cdot\frac{\partial G}{\partial x^i}   + \frac{\partial G}{\partial x^k} \cdot \left( \frac{\partial^2 G}{\partial x^j \partial x^i}\right) \right] dx^j \wedge dx^i \\
    &= \sum_{i=1}^n \sum_{j=1}^i \left( \frac{\partial^2 G_i}{\partial x^j \partial x^k} - \frac{\partial^2 G_j}{\partial x^k \partial x^i}\right) dx^i \wedge dx^j
  \end{align*}
  So $d\omega_k=0$ if and only if
  \[ \left( \frac{\partial^2 G_i}{\partial x^j \partial x^k} - \frac{\partial^2 G_j}{\partial x^k \partial x^i}\right) = 0 \qquad 1 \leq i,j,k \leq n \]
  again, by the Poincare lemma, if $d\omega_k =0$, then there exists a 0-form (smooth function) $F_k$ such that $\nabla F_k = \omega_k$. In this case, by Stokes theorem, we have 
  \[ \int_0^1 [DG_{sx}]^TDG_{sx} x ds = \int_0^1 \begin{pmatrix} \omega_1 \\ \vdots \\ \omega_d \end{pmatrix} = \int_{[0,1]} \begin{pmatrix} dF_1 \\ \vdots \\ dF_d \end{pmatrix} =  \begin{pmatrix} F_1 \\ \vdots \\ F_d \end{pmatrix} \Bigg \vert_0^1 = \begin{pmatrix} F_1(x) - F_1(0) \\ \vdots \\ F_d(x) - F_d(0) \end{pmatrix}   \]
  so setting $F = (F_1,\cdots,F_d)$ we have a vector field that satisfies 
  \[ DF_x = [DG_x]^T [DG_x]\]
\end{proof}

\section{Proof of Theorem 3}
\begin{proof}
Assume there exists $H$ that convexifies $G$. Then by construction, the Jacobian of $H \circ G$ takes the form $D(H \circ G) = [DG]^T_x [DG]_x$. But this means that by Stokes' theorem \eqref{eq:int_model} becomes
\[ \int_0^1 [DG]^T_{sx} [DG]_{sx} ds = \int_0^1 D(H \circ G)_{sx} x ds = H \circ G(sx) \big \vert_0^1 = H \circ G (x) - H\circ G (0) = H\circ G(x)\]
Where we without loss of generality assume $H \circ G (0) = 0$.
\end{proof}

\section{Further Directions}
\subsection{Additive Structure}

\label{apx:additive}
An alternative to developing deeper networks might be to attempt to add solutions that satisfy \eqref{eq:pde}. The construction \eqref{eq:int_model} is highly non-linear, so computing \eqref{eq:int_model} for a sum may be enough to significantly increase the complexity of our model.

It is easy to check that for two solutions $F$ and $G$, 
\begin{gather*} \frac{\partial^2 (G+F)}{\partial x^k \partial x^i} \cdot \frac{\partial (G+F)}{\partial x^j} = \frac{\partial^2 (G+F)}{\partial x^k \partial x^j} \cdot \frac{\partial (G+F)}{\partial x^i} \qquad \forall 1 \leq i,j,k \leq n
  \end{gather*}
 reduces to 
 \begin{gather*} \frac{\partial^2 G}{\partial x^k \partial x^i} \cdot \frac{\partial F}{\partial x^j} + \frac{\partial^2 F}{\partial x^k \partial x^i} \cdot \frac{\partial G}{\partial x^j}= \frac{\partial^2 G}{\partial x^k \partial x^j} \cdot \frac{\partial F}{\partial x^i} + \frac{\partial^2 F}{\partial x^k \partial x^j} \cdot \frac{\partial G}{\partial x^i}
  \end{gather*}
  so if we could enforce the condition that 
  \[ \frac{\partial^2 G}{\partial x^k \partial x^i} \cdot \frac{\partial F}{\partial x^j} = \frac{\partial^2 F}{\partial x^k \partial x^i} \cdot \frac{\partial G}{\partial x^j}  \qquad \forall i,j,k : i \neq j\]
  then $F+G$ would also be a solution to \eqref{eq:pde}. If we could enforce this, in theory we could add up to $n$ solutions in $\R^n$ in this manner, which may be of use in high dimensions. Of course, it's not obvious how to do this in the case where $F,G$ are neural networks, but it is possible a slight modification of the structure would be sufficient.

\subsection{Different Products}
\label{apx:prod}
While we have focused on structuring the integrand structured via the Gram product $[DG]^T DG$ (due to the chain rule interpretation), we could consider other products. The construction changes, however, because for the two other cases below, we need the matrices we combine to be PSD symmetric to begin with. This means that the only obvious way to use another product, if we denote it $a(\cdot, \cdot)$, is to start with two convex gradients $F$,$G$ (we could model using the constructions in Example 2), 
and integrate their composition $a(DF,DG)$. A few options are:

\begin{itemize}
    \item The Hadamard product, i.e 
\[ DG \odot DF \text{ for } G=\nabla g \qquad F = \nabla f \]
could be used. The famous Schur product theorem tells us that the Hadamard product of any two PSD matrices remains PSD. \cite{horn1994matrix} Unfortunately, though, we lose the interpretation as a Jacobian of a composition that the Gram product yields. Furthermore we still have a similar problem in that not any $F$ and $G$ can be used while recovering a result similar to \ref{thm:jacfac}, in this case, the PDE we require the pair $F,G$ to satisfy takes the form 
\[ \frac{\partial G_k}{\partial x^i \partial x^k} \left( \frac{\partial F_k}{\partial x^j} - \frac{\partial F_k}{\partial x^i} \right) = \frac{\partial F_k}{\partial x^i \partial x^k} \left( \frac{\partial G_k}{\partial x^j} - \frac{\partial G_k}{\partial x^i} \right) \]
Like \eqref{eq:pde}, it is not easily to consider when a pair will satisfy this equation. We leave questions concerning this construction to future work.

\item The Kronecker product could be also be considered. It also has the nice property that the Kronecker product of two PSD matrices remains PSD. \cite{horn1994matrix} However, we still have the issue of making the rows of $DF \otimes DG$ closed as 1-forms. One can derive a similar PDE as above, but the system is even more complex and restrictive, so we omit it. 
\end{itemize}

\end{document}